\documentclass[fleqn, a4paper, 11pt]{article}
\usepackage{epsfig, amssymb, amsbsy, amsmath, amsthm, mathrsfs}
\usepackage{psfrag}
\newtheorem{Prop}{Proposition}

\newtheorem{Def}{Definition}

\newtheorem{Cor}{Corollary}

\DeclareMathOperator*{\argmin}{argmin}

\title{\Large Mathematical Interpretation between Genotype and Phenotype Spaces and
Induced Geometric Crossovers}
\author{{\small Yourim Yoon$^\dag$,~~ Yong-Hyuk Kim$^\ddag$,~~ Alberto Moraglio$^*$, and~ Byung-Ro Moon$^\dag$}\\
{\small $\dag$ School of Computer Science \& Engineering, Seoul National University}\\
{\small Sillim-dong, Gwanak-gu, Seoul, 151-744, Korea}\\
{\small Email: {\tt \{yryoon, moon\}@soar.snu.ac.kr}}\\
{\small $\ddag$ Department of Computer Science \& Engineering, Kwangwoon University}\\
{\small 447-1, Wolgye-dong, Nowon-gu, Seoul, 139-701, Korea}\\
{\small Homepage: {\tt http://soar.snu.ac.kr/\~{ }yhdfly}}\\
{\small Email: {\tt yhdfly@kw.ac.kr}}\\
{\small $*$ Department of Computer Engineering, University of Coimbra}\\
{\small Polo II - Pinhal de Marrocos, Coimbra, 3030-290, Portugal}\\
{\small Email: {\tt moraglio@dei.uc.pt}}
}

\date{March 10, 2009}

\begin{document}


\maketitle

\begin{abstract}
In this paper, we present that genotype-phenotype mapping can be theoretically
interpreted using the concept of quotient space in mathematics.
Quotient space can be considered as mathematically-defined phenotype space in the
evolutionary
computation theory.
The quotient geometric crossover has the effect of reducing the search space
actually
searched by geometric crossover, and it introduces problem knowledge in the
search by using a distance better tailored to the specific solution
interpretation.
Quotient geometric crossovers are directly applied to the genotype space
but they have the effect of the crossovers performed on phenotype space.
We give many example applications of the quotient geometric crossover.
\end{abstract}

\section{Introduction}

In evolutionary computation, 
genotype means solution representation, which is the structure that can be
stored in a computer and manipulated. Phenotype means solution
itself without any reference to how it is represented. Sometimes it
is possible to have a one-to-one mapping between genotypes and phenotypes,
so the distinction between genotype and phenotype becomes purely
formal. However, in many interesting cases, phenotypes cannot be
represented uniquely by genotypes. So the same phenotype is
represented by more than one genotype. In such case we say that we
have a \emph{redundant representation}. For example, to represent a
graph we need to label its nodes and then we can represent it using
its adjacency matrix. This representation is redundant because the same
graph can be represented with more than one adjacency matrix by
relabeling its nodes.

There are quite a few problems in that it is hard to represent one phenotype
by just one genotype using traditional representations.
Roughly speaking, redundant representation leads to severe loss of search power
in genetic algorithms, in particular, with respect to traditional crossovers
\cite{CM03}.
To alleviate the problems caused by redundant representation,
a number of methods such as adaptive crossover have been proposed \cite{DH98,
Las91, Muhl92, HNG02}.
Among them, a technique called {\em normalization}\footnote{The term of 
{\em normalization} is firstly appeared in \cite{KM00-2}.
However, it is based on the adaptive crossovers proposed in \cite{Las91, Muhl92}.}
is representative. It transforms the genotype of a parent to another genotype
to be consistent
with the other parent so that the genotype contexts of the parents are as
similar as possible in crossover.
There have been a number of successful studies using normalization.
An extensive survey about normalization is appeared in \cite{Choi08}.

We recognized that genotype-phenotype mapping can be theoretically
interpreted using the concept of quotient space in mathematics.
In this paper, we formally present the general relation between the notion of quotient and
genotype-phenotype mapping, and we study the relation between genotype and
phenotype spaces and geometric crossovers on them.

For analysis, we adopted the concept of geometric crossover \cite{MorP05-p}
because it is representation-independent and
well-defined once a notion of
distance in the search space is defined.
In this study, we consider only genotype and phenotype spaces that are metric spaces.
So the metric for a space is considered as the most important
characteristic of its structure.
This approach enables to deal with the problem spaces more
mathematically.

The remainder of the paper is organized as follows.
In Section~\ref{sec:preliminaries}, we preliminarily present some
necessary mathematical notions and the
geometric framework. In Section \ref{sec:qgx}, the new notion of quotient
geometric crossover is introduced in connection with genotype-phenotype mapping.
In Section \ref{sec:applications}, we study several useful examples.
In Section \ref{sec:grouping} and \ref{sec:graphs},
we show how previous work on
groupings \cite{KYMM2006} and graphs
can be recast and understood more simply in terms of
quotient geometric crossover. In such problems, quotient geometric crossover
has the effect of filtering out inherent redundancy in the solution
representation.
In Section \ref{sec:symmetric},
we consider symmetric functions and their problem spaces.
The usage of the quotient
geometric crossover
for circular permutation encodings is discussed in Section \ref{sec:permutation}.
In Section \ref{sec:sequences},
we show how homologous
crossover for variable-length sequences \cite{moraglio2006a-p}
can be understood as a quotient geometric crossover.
Finally we give our conclusions in Section~\ref{sec:conclusions}.

\section{Preliminaries}
\label{sec:preliminaries}

\subsection{Mathematical Notions}

In the following, we give
some known mathematical definitions required to present our idea.

Given a set $X$ and an equivalence relation $\sim$ on $X$, the
equivalence class of an element $a$ in $X$ is the subset of all elements in $X$ that
are equivalent to $a$:

$$\bar{a} = \{ x \in X : a \sim x \}.$$

The set of all equivalence classes in $X$ given an
equivalence relation $\sim$ is usually denoted by $X/\!\!\sim$ and called the quotient set
of $X$ by $\sim$. This operation can be thought (very informally) as the act of
``dividing'' the input set by the equivalence relation. 
The quotient set is considered
as the set with all the equivalent points identified as a point.

Next, {\em group} \cite{AbstractAlgebra} is introduced.
A group is an algebraic structure consisting of a set together with an
operation that combines any two of its elements to form a third element.
To qualify as a group, the set and operation must satisfy a few conditions
called group axioms, namely associativity, identity, and invertibility.
Formally it is defined as follows:

\begin{Def}[Group]
A {\em group} $(\mathcal{G}, *)$ is a set $\mathcal{G}$ closed under a binary
operation $*$, such that the following axioms are satisfied:\\
(i) Associativity: for all $a,b,c \in G$, we have
$$(a*b)*c = a*(b*c).$$
(ii) Identity:
there is an element $e$ in $G$ such that for all $x \in G$,
$$e*x=x*e=x.$$
(iii) Invertibility:
for each $a \in G$, there is an element $a^{-1}$ in $G$ such that
$$a*a^{-1}=a^{-1}*a=e.$$
\end{Def}

\noindent In this paper, we will use groups for
constructing equivalence relations with {\em good} properties.

In the following, we view the problem spaces as metric spaces.
It is a reasonable assumption because
the general solution spaces usually have metrics.
For example, binary space has the Hamming distance and
real space has Minkowski distances including the Euclidean distance.
Formally, the term \emph{metric} - or \emph{distance} - denotes any real-valued
function that conforms to the axioms of identity, symmetry, and
triangular inequality. Now, we introduce an isometry on a metric space.

\begin{Def}[Isometry]
Let $(X,d)$ be a metric space. If $f\!\!: X \rightarrow X$
satisfies the condition
$$d(f(x),f(y)) = d(x,y)$$
for all $x,y \in X$, then $f$ is called an {\em isometry} of $X$.
\end{Def}

\noindent The set of isometries on $X$ is denoted by $Iso(X)$.
$Iso(X)$ forms a group under function composition operator.
In our study, an isometry subgroup $\mathcal{G} \subseteq Iso(X)$ 
will be considered to generate an equivalence relation for quotient metric space.

\subsection{Geometric Preliminaries}

In this subsection we provide
some geometric definitions, which extend those
introduced in \cite{MorP04-p, MorP05-2-p}. The following
definitions are taken from \cite{Deza}.

In a metric space $(X,d)$, a \emph{line segment} (or closed interval)
is the set of the form $[x;y]_d = \{z \in X ~|~ d(x,z) + d(z,y) =
d(x,y)\}$, where $x,y \in X$ are called extremes of the segment.
Metric segment generalizes the familiar notions of segment in the
Euclidean space to any metric space through distance redefinition.
Notice that a metric segment does not coincide to a shortest path
connecting its extremes (\emph{geodesic}) as in an Euclidean space.
In general, there may be more than one geodesic connecting two
extremes; the metric segment is the union of all geodesics.

We assign a structure to the solution set $X$ by endowing it with a
notion of distance $d$. $M=(X, d)$ is therefore a solution
\emph{space} and $(M, f)$ is the corresponding fitness landscape,
where $f$ is the fitness function over $X$.

\subsection{Geometric Crossover}

Geometric crossover is a representation-independent search operator 
that generalizes many
pre-existing search operators for the major representations used in
evolutionary algorithms, such as binary strings \cite{MorP04-p},
real vectors \cite{MorP04-p, Yoon07}, permutations \cite{MorP05-p},
permutations with repetition \cite{Mor07},
syntactic trees \cite{MorP05-2-p}, sequences
\cite{moraglio2006a-p}, and sets \cite{Mor06}. 
It is defined in geometric terms using
the notions of line segment and ball. These notions and the
corresponding genetic operators are well-defined once a notion of
distance in the search space is defined. Defining search operators
as functions of the search space is opposite to the standard way
\cite{Jones95} in which the search space is seen as a function of
the search operators employed. This viewpoint greatly simplifies the
relationship between search operators and fitness landscape and has
allowed us to give simple {\em rules-of-thumb} to build crossover
operators that are likely to perform well.

The following definitions are \emph{representation-independent}
therefore applicable to any representation.

\begin{Def}[Image set]
The \emph{image set} $Im[OP]$ of a genetic operator $OP$ is the set
of all possible offspring produced by $OP$.
\end{Def}

\begin{Def}[Geometric crossover]
A binary operator $GX$ is a {\em geometric crossover} under the metric $d$ if
all offspring are in the segment between its parents $x$ and $y$, i.e.,
$$Im[GX(x,y)] \subseteq [x;y]_d.$$
\end{Def}

\noindent A number of general properties for geometric crossover
have been derived in \cite{MorP04-p} where it was also shown that
traditional mask-based crossovers are geometric under the Hamming distance.
Moraglio and Poli also studied various crossovers for
permutations, revealing that PMX (partially matched crossover) 
\cite{Gold85}, a well-known crossover for
permutations, is geometric under swap distance. Also, they found that
cycle crossover \cite{Oliver87}, another traditional crossover for permutations, is
geometric under swap distance and under the Hamming distance.

Theoretical results of metric spaces can naturally lead to
interesting results for geometric crossover. In particular,
Moraglio and Poli showed that
the notion of \emph{metric transformation} has great potential for
geometric crossover in \cite{MorP2006}.
A metric transformation is an operator that
constructs new metric spaces from pre-existing metric spaces: it
takes one or more metric spaces as input and outputs a new metric
space. The notion of metric transformation becomes extremely
interesting when considered together with distances firmly rooted in
the syntactic structure of the underlying solution representation
(e.g., edit distance). In these cases it gives rise to a simple and
\emph{natural interpretation in terms of syntactic transformations}.

Moraglio and Poli extended
the geometric framework introducing the notion of product crossover
associated with the Cartesian product of metric spaces in \cite{MorP2006}.
This is a very important tool that allows one to build new geometric
crossovers customized to problems with mixed representations by
combining pre-existing geometric crossovers in a straightforward
way. Using the product geometric crossover, they also showed that
traditional crossovers for symbolic vectors and blend crossovers for
integer and real vectors are geometric crossover.

\section{Quotient Geometric Crossover}
\label{sec:qgx}

\subsection{Motivation}

Geometric operators are defined as functions of the distance
associated to the search space. However, the search space does not
come with the problem itself. The problem consists only of a fitness
function to optimize, that defines what a solution is and how to
evaluate it, but it does not give any structure on the solution set.
The act of putting a structure over the solution set is a part of the
search algorithm design and it is a designer's choice.

A fitness landscape is the fitness function plus a structure over
the solution space. So, for each problem, there is one fitness
function but as many fitness landscapes as the number of possible
different structures over the solution set. In principle, the
designer can choose the structure to assign to the solution set
completely independently from the problem at hand. However, because
the search operators are defined over such a structure, doing so
would make them  decoupled from the problem at hand, hence turning
the search into something very close to random search.

To avoid such problem, one can exploit problem knowledge in the
search. It can be achieved by carefully designing the connectivity
structure of the fitness landscape. For example, one can study the
objective function of the problem and select a neighborhood
structure that couples the distance between solutions and their
fitness values. Once it is done, problem knowledge can be exploited
by search operators to perform better than random search, even if
the search operators are problem-independent (as is the case of
geometric operators). Indeed, the fitness landscape is
a knowledge interface between the problem at hand and a formal,
problem-independent search algorithm.

Under which conditions is a landscape  well-searchable by geometric
operators? As a rule of thumb, geometric
crossover works well on landscapes where the closer pairs of
solutions, the more correlated their fitness values. Of course this
is no surprise: the importance of landscape smoothness has been
advocated in many different contexts and has been confirmed in
uncountable empirical studies with many neighborhood search
meta-heuristics \cite{pardalos2002}. We operate according to the
following rules-of-thumb:

\noindent \emph{Rule-of-thumb 1}: if we have a good distance for the
problem at hand, then we have a good geometric crossover.

\noindent \emph{Rule-of-thumb 2}: a good distance for the problem at
hand is a distance that makes the landscape ``smooth.''

\subsection{Genotype-Phenotype Mapping}


We formally present the general relation between the notion of quotient and
genotype-phenotype mapping. 
Let $G$ and $P$ be genotype space and phenotype one, respectively.
Consider a genotype-phenotype mapping $\mathfrak{g}\!\!: G \rightarrow P$
that are not injective (i.e., redundant representation). The mapping 
$\mathfrak{g}$ induces a natural equivalence relation $\sim$ on the set of
genotypes: {\em genotypes with the same phenotype belong to the same
class}. Then the phenotype space $P$ becomes exactly a quotient space $G/\!\!\sim$
of the genotype space $G$.

The advantage of geometric crossover is that
we can formally define a geometric crossover under the distance once a distance
is defined. Then what if the quotient space $G/\!\!\sim$ has a distance $d_P$ induced by the
distance $d_G$ of $G$? If so, the geometric crossover under $d_P$ would be
a natural crossover since it reflects the structure of the genotype space $G$
by involving the distance $d_G$ of $G$.

By applying the formal definition of geometric crossover to the
metric spaces $(G,d_G)$ and $(P,d_P)$, we obtain the geometric
crossovers $GX_G$ and $GX_P$, respectively. $GX_G$ searches the space of
genotypes and $GX_P$ searches that of phenotypes. Searching the
space of phenotypes has a number of advantages: (i) it is smaller
than the space of genotypes, hence quicker to search (ii) the
phenotypic distance is better tailored to the underlying problem,
hence the corresponding geometric crossover works better (iii) the
space of phenotypes has different geometric characteristics from the
genotypic space. It can be used to remove unwanted bias from
geometric crossover.

However, the crossover $GX_P$ cannot be directly implemented because
it recombines phenotypes that are objects that cannot be directly
represented. 
So, we propose a notion of 
quotient geometric crossover to search
the space of phenotypes with the crossover $GX_P$ \emph{indirectly}
by manipulating the genotypes $G$.
The relationship among $G$, $P$, and their geometric crossovers
is illustrated through a diagram in Figure~\ref{fig:diagram}.


\begin{figure}
\psfrag{g}{\large $\mathfrak{g}$}
\psfrag{geometric}{geometric}
\psfrag{quotient}{quotient}
\psfrag{crossover}{crossover}
\centering
\centerline{\epsfig{figure=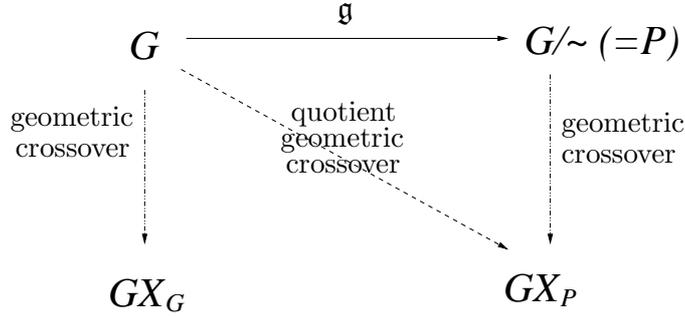,width=3.5in,angle=0}}
\caption[]{Diagram linking genotype, phenotype spaces, and geometric crossovers} 
\label{fig:diagram}
\end{figure}

In the next subsection we present more formally the concept of
quotient metric space and quotient geometric crossover in relation with
genotype-phenotype mapping.

\subsection{Quotient Metric Space}

Let $(X, d)$ be a metric space
and $(\mathcal{G},\cdot) \subseteq Iso(X)$ be a subgroup of the isometry
group, where $\cdot$ means the function composition operator.
We introduce a relation $\sim_{\mathcal{G}}$: $x$ and $y$ in $X$ are
equivalent if and only if $x=g(y)$ for some $g \in \mathcal{G}$.
Then $\sim_{\mathcal{G}}$ is an equivalence relation by the following
proposition.

\begin{Prop}
\label{prop:equivalence}
Relation $\sim_\mathcal{G}$ is an equivalence relation.
\end{Prop}

\begin{proof}
Assume that $x$, $y$, and $z \in X$.\\
(i) Reflexivity:
(Since $\mathcal{G}$ is a group, identity map $e$ is in $\mathcal{G}$.)
Since $x = e(x)$, $x \sim_{\mathcal{G}} x$.\\
(ii) Symmetry:
Suppose that $x \sim_G y$, i.e., $x=g(y)$ for some $g \in G$.
There exists $g^{-1} \in G$ since $G$ is a group.
Then, $y = g^{-1}(x)$. So $y \sim_G x$. \\
(iii) Transitivity:
Suppose that $x \sim_G y$ and $y \sim_G z$.
$x=g(y)$ and $y=h(z)$ for some $g$ and $h \in G$.
Then $x = g(y) = g(h(z)) = (g \cdot h)(z)$. $g \cdot h$ is in $G$
since $G$ is a group. Hence, $x \sim_{\mathcal{G}} z$.
\end{proof}

For $x \in X$, equivalence class $\bar{x}$ can be written as 
$\bar{x} = \{g(x) : g \in \mathcal{G}\}$.
Now we will give a metric on
$X/\!\!\sim_G$ (usually denoted by just $X/\mathcal{G}$) induced by the original metric $d$ on $X$.

\begin{Def}[Quotient metric]
\label{def:quotient_metric}
{\em Quotient metric} $\bar{d}(\bar{x},\bar{y})$ is defined as $\min\{d(x',y') : x' \in \bar{x}, y' \in \bar{y} \}$.
\end{Def}

\noindent It is shown that $\bar{d}$ is actually a metric on $X/\mathcal{G}$ in the following
proposition.

\begin{Prop}
\label{prop:metric}
$(X/G,\bar{d})$ is a metric space, i.e., $\bar{d}$ is a metric in $X/G$.
\end{Prop}

\begin{proof}
Assume that $x$, $y$, and $z \in X$.\\
(i) Identity:
$0 \le \bar{d}(\bar{x}, \bar{x}) \le d(x, x) = 0$.\\
(ii) Symmetry:
There exist $x_1 \in \bar{x}$ and $y_1 \in \bar{y}$ such that
$\bar{d}(\bar{x},\bar{y})=d(x_1,y_1)$.
Then,
$\bar{d}(\bar{x},\bar{y}) = d(x_1, y_1) = d(y_1,x_1) \ge
\bar{d}(\bar{y},\bar{x})$.
Similarly, $\bar{d}(\bar{y}, \bar{x}) \ge \bar{d}(\bar{x}, \bar{y})$.
Hence, $\bar{d}(\bar{x},\bar{y}) = d(x_1,y_1)$.\\
(iii) Triangular inequality:
There exist $x_1 \in \bar{x}$ and $y_1 \in \bar{y}$ such that
$\bar{d}(\bar{x},\bar{y})=d(x_1,y_1)$. Also, There exist $y_2 \in \bar{y}$ and
$z_2 \in \bar{z}$ such that
$\bar{d}(\bar{y},\bar{z})=d(y_2,z_2)$.
Since $y_1$ and $y_2$ belong to the same equivalence class,
there exists $g \in G$ such that $y_1 = g(y_2)$.
Then,
\begin{eqnarray*}
&   & \bar{d}(\bar{x}, \bar{y}) + \bar{d}(\bar{y},\bar{z})\\
& = & d(x_1, y_1) + d(y_2,z_2))\\
& = & d(x_1, g(y_2)) + d(y_2, z_2)\\
& = & d(x_1, g(y_2)) + d(g(y_2),g(z_2))~~ (\because g \in Iso(X).)\\
& \ge & d(x_1, g(z_2))~~~(\because d \textrm{ is a metric in }X.)\\
& \ge & \bar{d}(\bar{x}, \bar{z}). ~~(\because z_2 \sim_{\mathcal{G}} g(z_2).)
\end{eqnarray*}
\end{proof}

\noindent The following proposition gives a simpler definition of quotient metric
$\bar{d}$.

\begin{Prop}
If we let $\tilde{d}(\bar{x},\bar{y}) := \min\{d(x,y') : y' \in \bar{y} \}$,
$\tilde{d}(\bar{x},\bar{y}) = \bar{d}(\bar{x},\bar{y})$.
\end{Prop}

\begin{proof}
Let $\bar{x},\bar{y} \in X/G$.
It is clear that $\tilde{d}(\bar{x},\bar{y}) \ge \bar{d}(\bar{x},\bar{y})$
by the definition.
Now we will show that
$\tilde{d}(\bar{x},\bar{y}) \le \bar{d}(\bar{x},\bar{y})$.
Suppose that
$\bar{d}(\bar{x},\bar{y})=d(x_1,y_1)$.
Since $x$ and $x_1$ belong to the same equivalence class,
there exists $g \in G$ such that $x = g(x_1)$.
Since $g$ is an isometry, 
$d(x_1,y_1) = d(g(x_1),g(y_1)) = d(x,g(y_1))$.
$g(y_1) \sim_G y_1 \sim_G y$. So $\tilde{d}(\bar{x},\bar{y}) \le
d(x_1,y_1) = \bar{d}(\bar{x},\bar{y})$.
\end{proof}

This metric space $(X/\mathcal{G}, \bar{d})$ is called {\em quotient metric space}.
Quotient space conceptually corresponds to the phenotype space.
The line segment and the geometric crossover in the quotient metric space
are defined in the same way as in other metric spaces.
However, since in general cases solutions are represented only 
in the genotype space,
we need to define line segments and crossovers on $(X, d)$,
not on $(X/\mathcal{G}, \bar{d})$, to practically apply the concept.

In a metric space $(X, d)$, a \emph{quotient line segment}
is the set of the form 
$[x;y]_{\bar{d}} = \{z \in X ~|~ \bar{d}(\bar{x},\bar{z})+
\bar{d}(\bar{z},\bar{y}) = \bar{d}(\bar{x},\bar{y}), \bar{z} \in X/G\}$,
where $\bar{x},\bar{y} \in X/G$.

\begin{Prop}
\label{prop:segment}
If $\bar{d}(\bar{x},\bar{y}) = d(x,y^*)$,
$[x;y^*]_d \subseteq [x;y]_{\bar{d}}$.
\end{Prop}

\begin{proof}
Let $z \in [x;y^*]_d$.
Then,
$\bar{d}(\bar{x},\bar{z}) + \bar{d}(\bar{z},\bar{y})
\le d(x,z) + d(z,y^*)
= d(x,y^*) = \bar{d}(\bar{x},\bar{y})$.
Since $\bar{d}(\bar{x},\bar{z}) + \bar{d}(\bar{z},\bar{y}) \ge \bar{d}(\bar{x},\bar{y})$
by the property of triangular property, 
$\bar{d}(\bar{x},\bar{z}) + \bar{d}(\bar{z},\bar{y}) = \bar{d}(\bar{x},\bar{y})$.
So, $z \in [x;y]_{\bar{d}}$.
\end{proof}

Now we can define the {\em quotient geometric crossover}.

\begin{Def}[Quotient geometric crossover]
A binary operator $GX_q$ is a {\em quotient geometric crossover}
under the metric $d$ and the equivalence relation $\sim_G$ if
all offspring are in the quotient line segment between its parents $x$ and $y$,
i.e., $GX_q(x,y) \subseteq [x;y]_{\bar{d}}$.
\end{Def}

In the following we define the {\em induced quotient crossover}, which
is a kind of quotient geometric crossovers.
This crossover is defined using the original geometric crossover under the
original distance $d$.
This crossover is not only concrete but also easily implemented
while the quotient geometric crossover is conceptual.

\begin{Def}[Induced quotient crossover]
\label{def:iqx}
First, find $y^*$ in the equivalence class $\bar{y}$ of the second parent $y$
such that $\bar{d}(\bar{x},\bar{y}) = d(x,y^*)$.
Then, do the geometric crossover on $X$ using the first parent $x$ and the normalized
second parent $y^*$.
\end{Def}

\begin{Cor}
Induced quotient crossover is a quotient geometric crossover.
\end{Cor}

\begin{proof}
Let $y^* \in \bar{y}$ be a normalized second parent i.e.,
$\bar{d}(\bar{x},\bar{y}) = d(x,y^*)$.
Then, $[x;y^*]_d \subseteq [x;y]_{\bar{d}}$ by Proposition~\ref{prop:segment}.
This satisfies the definition of quotient geometric crossover.
\end{proof}

\noindent Induced quotient crossover can be a bridge between
the original geometric crossover and the quotient geometric crossover.
We can redraw Figure~\ref{fig:diagram} including the
induced quotient crossover.
It is shown in Figure~\ref{fig:newdiagram}.

\begin{figure}
\psfrag{g}{\large $\mathfrak{g}$}
\psfrag{geometric}{geometric}
\psfrag{quotient}{quotient}
\psfrag{crossover}{crossover}
\psfrag{induced quotient}{induced quotient}
\centering
\centerline{\psfig{figure=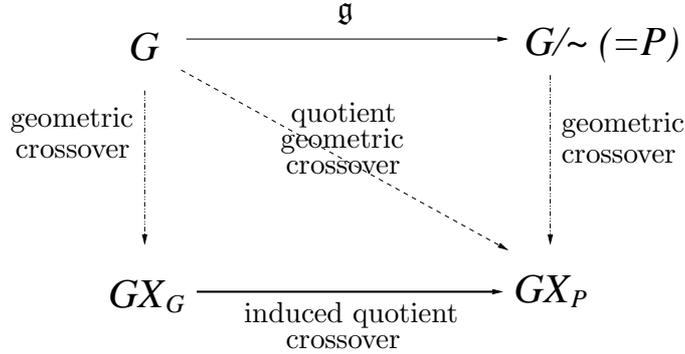,width=3.5in,angle=0}}
\caption[]{Diagram linking genotype, phenotype spaces, and crossovers including
induced quotient crossover}
\label{fig:newdiagram}
\end{figure}

In the next section, we consider a number of equivalence
classes for the quotient operation and its related induced genotypic
crossover transformation.

\section{Examples}
\label{sec:applications} 

In this section, we introduce examples of quotient geometric crossovers.
In some examples (groupings in Subsection~\ref{sec:grouping} and sequences in Subsection \ref{sec:sequences}), 
previously proposed crossovers are reinterpreted as quotient geometric
crossovers. In other examples, we introduce new crossovers 
consistent with genotype-phenotype mapping
of given problems using the concept
of quotient geometric crossover.

\subsection{Groupings}
\label{sec:grouping}

{\em Grouping} problems \cite{Falk98} are commonly
concerned with partitioning given item set into mutually disjoint
subsets. Examples belonging to this class of problems are multiway
graph partitioning, graph coloring, bin packing, and so on. Grouping
representation is also used to solve the joint replenishment
problem, which is a well-known problem appeared in the field of
industrial engineering \cite{Olsen05}. In this class of problems,
the normalization decreased the problem difficulty and led to
notable improvement in performance.

Most normalization studies for grouping problems were focused on the $k$-way
partitioning problem. In the problem, the $k$-ary representation, in which $k$
subsets are represented by the integers from $0$ to $k - 1$, has been generally
used. In this case, one phenotype (a $k$-way partition) is represented by $k!$
different genotypes. In the problem, a normalization method was used in
\cite{KM00-2}.
Other studies for the $k$-way partitioning problem used the same technique
\cite{CM03, JPKim01}.
In the sense that normalization pursues the minimization of genotype inconsistency
among chromosomes, in \cite{KimM05-p}, Kim and Moon proposed
an optimal, efficient normalization method for grouping problems and
a distance measure, the {\em labeling-independent distance}, that eliminates
such dependency completely.

Now we reinterpret the previous work in terms of quotient space.
Let $a, b \in X = \{1, 2, \ldots, k\}^n$
be $k$-ary encodings (fixed-length vectors on a $k$-ary alphabet)
and $\Sigma_k$ be a set of all permutations of length $k$.
For each $\sigma \in \Sigma_k$, we can view
$\sigma$ as a function on $X$ by defining $\sigma(a)$ be
a permuted encoding of $a$ by a
permutation $\sigma$.
For example, in the case that $a = (1,2,3,3,2,4,1,4)$ is 
a $4$-ary encoding and $\sigma = {1~2~3~4 \choose 2~4~3~1}
\in \Sigma_4$, $\sigma(a) = (2,4,3,3,4,1,2,1)$.

It is well known that permutations form a group. Hence,
$\Sigma_k$ is a group. Moreover, when we use the Hamming distance
$H$ on $X$, it is easy to check that each $\sigma \in \Sigma_k$ is
an isometry.
So $\sim_{\Sigma_k}$ becomes an equivalence relation and
the quotient metric in Definition~\ref{def:quotient_metric} is well defined
by Proposition~\ref{prop:metric}.
The quotient metric was introduced as labeling-independent distance in \cite{KimM05-p}.
In the context of this study, it is rewritten as follows:
\begin{displaymath}
\bar{d}(\bar{a},\bar{b}) := \min_{\sigma \in \Sigma_k} H(a,\sigma(b)).
\end{displaymath}

\noindent An example case is shown in Figure~\ref{fig:ex_grouping}.

\begin{figure}
\centering
$$X=\{1,2,3\}^4$$
$$x = (1,2,3,1), y = (2,1,2,3) \in X$$
\begin{tabular}{|c|c|c|}
\hline
$\Sigma_k$ & $\bar{y}$ & $H(x,\sigma(y))$ \\
\hline
$\sigma_1 = (1,2,3)$ & $\sigma_1(y) = (2,1,2,3)$ & $H(x,\sigma_1(y)) = 4$ \\
$\sigma_2 = (1,3,2)$ & $\sigma_2(y) = (3,1,3,2)$ & $H(x,\sigma_2(y)) = 3$ \\
$\sigma_3 = (2,1,3)$ & $\sigma_3(y) = (1,2,1,3)$ & $H(x,\sigma_3(y)) = 2$ \\
$\sigma_4 = (2,3,1)$ & $\sigma_4(y) = (3,2,3,1)$ & $H(x,\sigma_4(y)) = 1$ \\
$\sigma_5 = (3,1,2)$ & $\sigma_5(y) = (1,3,1,2)$ & $H(x,\sigma_5(y)) = 3$ \\
$\sigma_6 = (3,2,1)$ & $\sigma_6(y) = (2,3,2,1)$ & $H(x,\sigma_6(y)) = 3$ \\
\hline
\end{tabular}
$$y^*=\sigma_4(y) = (3,2,3,1)$$
$$\bar{d}(\bar{x},\bar{y}) = 1$$
\caption[]{An example of grouping}
\label{fig:ex_grouping}
\end{figure}

The definition of labeling-independent crossover presented in \cite{KYMM2006}
is in the following.

\begin{Def}[Labeling-independent crossover]
Normalize the second parent to the first under the Hamming distance.
Then, do the normal crossover using the first parent and the normalized
second parent.
\end{Def}

\noindent This crossover is exactly the induced quotient crossover.
For the process of normalization, it is possible to
enumerate all $k!$ permutations and find an optimal
one among them. However, for a large $k$, such a procedure
is intractable.
Fortunately, it can be done in $O(k^3)$ time using 
the {\em Hungarian method} proposed by Kuhn \cite{Kuhn55}.

In summary, we have the following. \\

\begin{tabular}{|l|l|}
\hline
Example & Groupings \\ \hline \hline
Genotype space $X$ & $\{1,2,\ldots,k\}^n$ \\ \hline
Isometry group $\mathcal{G}$ & $\Sigma_k$: a set of all permutations\\
inducing phenotype space $X/\mathcal{G}$ &  of length $k$ \\ \hline
Metric $d$ on $X$ & Hamming distance $H$ \\ \hline
Original geometric crossover & traditional crossover for vectors\\ \hline
Induced quotient crossover & Labeling-independent crossover in \cite{KYMM2006} \\
\hline
\end{tabular}
\\

From understanding normalization for grouping problems in
terms of quotient geometric crossover, we can
understand the benefit of normalization in terms of landscape
analysis. 
We have already done this in our previous work \cite{Mor07}.

\subsection{Graphs}
\label{sec:graphs}

In this subsection, we consider any problem naturally defined over a graph in
which the
fitness of the solution does not depend on the labels on the nodes but
only on the structural relationship, i.e., edge between nodes.

Formally, let $A \in \mathfrak{M}_n$ be the adjacency matrix of a labeled graph
using labels of $n$ nodes and let $\mathcal{P}_n$ be a set of all $n \times n$
permutation matrices\footnote{
Permutation matrix is a $(0,1)$-matrix with exactly one $1$ in every row and
column.}.
Then, for each permutation matrix $P \in \mathcal{P}_n$, 
the matrix $PAP^T$ means the labeled graph obtained by relabeling $A$
according to the permutation represented by $P$.
The fitness $f\!: \mathfrak{M}_n \rightarrow \mathbb{R}$
satisfies that for every $A \in \mathfrak{M}_n$ and every permutation matrix
$P$,
$f(A)=f(PAP^T)$.

Let $(\mathfrak{M}_n,H)$ be a metric space on the labeled graphs under the
Hamming distance $H$.
Notice that this metric is labeling-dependent.
In particular, $H(A,PAP^T)$ may not be zero
although $A$ and $PAP^T$ represent the same structure.
If $A$ is equal to $PBP^T$ for some permutation matrix $P$,
we define $A$ and $B$ to be in relation $\sim_{\mathcal{P}_n}$, 
i.e., $A \sim_{\mathcal{P}_n} B$.
Since a set of permutation matrices $\mathcal{P}_n$ forms a group,
the relation $\sim_{\mathcal{P}_n}$ is an equivalence relation
by Proposition~\ref{prop:equivalence}.

The equivalence class $\bar{A}$ is represented as follows:
$$\bar{A} := \{PA : P \in \mathcal{P}_n\}.$$
It corresponds to an {\em unlabeled graph} and
the quotient space $\mathfrak{M}_n/\mathcal{P}_n$ can be understood as
{\em unlabeled-graph space}.
$\mathfrak{M}_n/\mathcal{P}_n$ is a quotient metric space
by Proposition~\ref{prop:metric}.
So we obtain induced quotient metric on $\mathfrak{M_n}/\mathcal{P}_n$.
It can be written as follows:
$$\bar{d}(\bar{A},\bar{B}) =\min_{P \in \mathcal{P}_n} H(A,PB).$$

\begin{figure}
\centering
$$X=\mathfrak{M}_3$$
$$ A = \left(\begin{array}{ccc}
0 & 1 & 0 \\
1 & 0 & 1 \\
0 & 1 & 0
\end{array}\right), B = \left(\begin{array}{ccc}
0 & 0 & 1 \\
0 & 0 & 1 \\
1 & 1 & 0
\end{array}\right) \in X$$
\centerline{\epsfig{figure=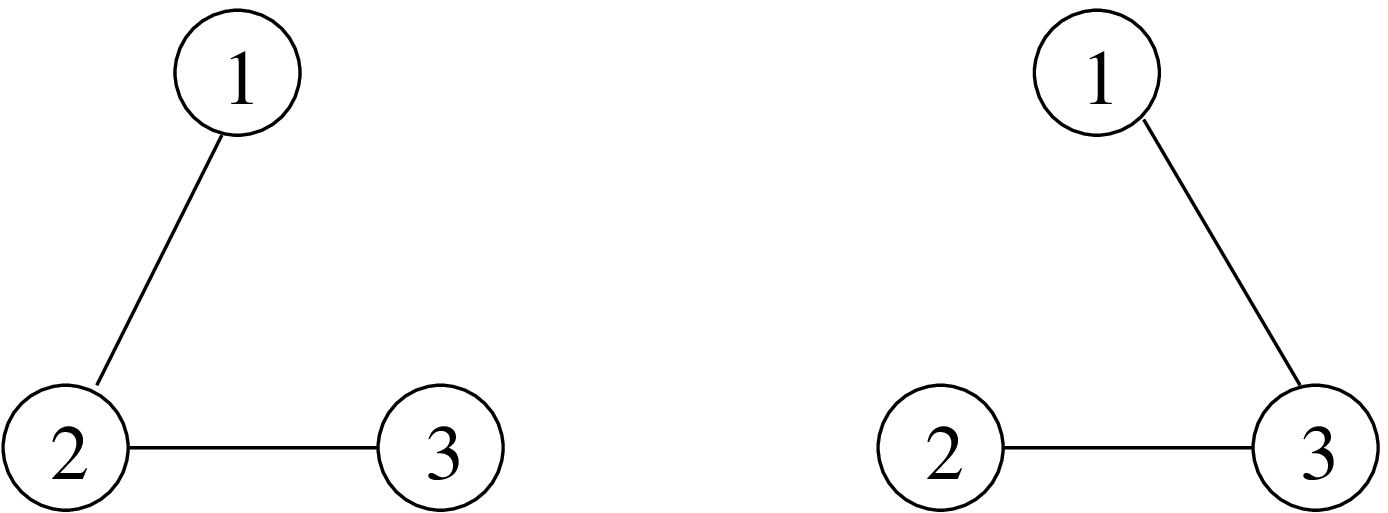,width=2.5in,angle=0}}
\centerline{~}
\begin{tabular}{|c|c|c|}
\hline
$\mathcal{P}_3$ & $\bar{B}$ & $H(A,PBP^T)$ \\
\hline
$P_1 = \left(\begin{array}{ccc}
1 & 0 & 0 \\
0 & 1 & 0 \\
0 & 0 & 1
\end{array}\right)$
& $P_1 B P_1^T = \left(\begin{array}{ccc}
0 & 0 & 1 \\
0 & 0 & 1 \\
1 & 1 & 0
\end{array}\right)$ & $H(A,P_1 B P_1^T) = 4$ \\ \hline
$P_2 = \left(\begin{array}{ccc}
1 & 0 & 0 \\
0 & 0 & 1 \\
0 & 1 & 0
\end{array}\right)$
& $P_2 B P_2^T = \left(\begin{array}{ccc}
0 & 1 & 0 \\
1 & 0 & 1 \\
0 & 1 & 0
\end{array}\right)$ & $H(A,P_2 B P_2^T) = 0$ \\ \hline
$P_3 = \left(\begin{array}{ccc}
0 & 1 & 0 \\
1 & 0 & 0 \\
0 & 0 & 1
\end{array}\right)$
& $P_3 B P_3^T = \left(\begin{array}{ccc}
0 & 0 & 1 \\
0 & 0 & 1 \\
1 & 1 & 0
\end{array}\right)$ & $H(A,P_3 B P_3^T) = 4$ \\ \hline
$P_4 = \left(\begin{array}{ccc}
0 & 1 & 0 \\
0 & 0 & 1 \\
1 & 0 & 0
\end{array}\right)$
& $P_4 B P_4^T = \left(\begin{array}{ccc}
0 & 1 & 0 \\
1 & 0 & 1 \\
0 & 1 & 0
\end{array}\right)$ & $H(A,P_4 B P_4^T) = 0$ \\ \hline
$P_5 = \left(\begin{array}{ccc}
0 & 0 & 1 \\
1 & 0 & 0 \\
0 & 1 & 0
\end{array}\right)$
& $P_5 B P_5^T = \left(\begin{array}{ccc}
0 & 1 & 1 \\
1 & 0 & 0 \\
1 & 0 & 0
\end{array}\right)$ & $H(A,P_5 B P_5^T) = 4$ \\ \hline
$P_6 = \left(\begin{array}{ccc}
0 & 0 & 1 \\
0 & 1 & 0 \\
1 & 0 & 0
\end{array}\right)$
& $P_6 B P_6^T = \left(\begin{array}{ccc}
0 & 1 & 1 \\
1 & 0 & 0 \\
1 & 0 & 0
\end{array}\right)$ & $H(A,P_6 B P_6^T) = 4$ \\ \hline
\end{tabular} \\
\vspace{0.2in}
$B^*= P_2 B P_2^T$ or $P_4 B P_4^T$
$$\bar{d}(\bar{A},\bar{B}) = 0$$
\caption[]{An example of graph}
\label{fig:ex_graph}
\end{figure}

\noindent An example for graphs is shown in Figure \ref{fig:ex_graph}.

Now we design induced quotient crossover.
In this example, the process of finding the normalized second parent $y^*$
can be understood as graph matching in terms of graphs which are not adjacency matrices.

\begin{Def}[Induced quotient crossover for graphs]
Do the graph matching of the second parent $B$ to the first $A$ under the
Hamming distance $H$,
i.e., $$B^* := \argmin_{B' \in \bar{B}} H(A, B').$$
Then, do the normal crossover using the first parent $A$ and the graph-matched
second parent $B^*$.
\end{Def}

\noindent The induced quotient crossover
is defined over unlabeled graphs
$\mathfrak{M}_n/\mathcal{P}_n$.
This space is much
smaller than labeled graphs $\mathfrak{M}_n$. More precisely,
$|\mathfrak{M}_n/\mathcal{P}_n|=|\mathfrak{M}_n|/n!$. This means that the more the
labels are, the
smaller the unlabeled-graph space is when compared with the labeled-graph space.
Smaller space means better performance, given the same amount of
evaluations.

Now we tell how to guide the implementation
using graph matching for specific geometric crossovers.
To implement the geometric crossover over unlabeled graphs, we need to use
labeled graphs.
The labeling results are necessary to represent and handle the solution,
even if in fact it is only an auxiliary function and can be considered
as not being part of the problem to solve. Graph matching before crossover
allows to implement the geometric crossover on the unlabeled-graph
space. We use the corresponding geometric crossover over the auxiliary
space of the labeled graph after graph matching. \\

\begin{tabular}{|l|l|}
\hline
Example & Graphs \\ \hline \hline
Genotype space $X$ & $\mathfrak{M}_n$ (the set of all $n \times n$ adjacency matrices) \\ \hline
Isometry group $\mathcal{G}$ & $\mathcal{P}_n$ : a set of all $n \times n$ permutation\\
inducing phenotype space $X/\mathcal{G}$ & matrices \\ \hline
Metric $d$ on $X$ & Hamming distance $H$ \\ \hline
Original geometric crossover & traditional crossover on adjacency matrices \\
& seen as length-$n^2$ vectors\\ \hline
Induced quotient crossover & graph matching before traditional crossover \\
& (newly introduced in this study) \\
\hline
\end{tabular}
\\

By applying the quotient geometric crossover on graphs,
we can design a crossover better tailored to graphs. The notion
of graph matching before crossover arises directly from the
definition of quotient geometric crossover. Graphs are very important because
they are ubiquitous. In future work we will test this crossover on some
applications. Graphs and groupings can be seen as particular cases of
labeled structures in which the fitness of a solution depends only
on the structure and not on the specific labeling. In future work we
will also study the class of labeled structures in combination with
quotient geometric crossover.

\subsection{Symmetric Functions}
\label{sec:symmetric}
A symmetric function on $n$ variables $x_1 , x_2, \ldots, x_n$ is a function that is
unchanged by any permutation of its variables. That is , if $f(x_1,x_2,\ldots, x_n)
= f(x_\sigma(1), x_\sigma(2), \ldots, x_\sigma(n))$ for any permutation $\sigma$,
the function $f$ is called symmetric function.
In this subsection, we consider problems of which fitness function is
symmetric. 
Some evolutionary studies have been made on such problems \cite{Seo08, Yoon08}.
More properties about specific symmetric functions 
are introduced in \cite{Choi07,EnumerativeCombinatorics}.

Solutions for symmetric functions are typically represented as
$n$-dimensional vectors, i.e., length-$n$ strings.
Let $X$ be the solution space (or domain) of given symmetric function
and $\Sigma_n$ be a set of all permutations of length $n$.
Similarly to the example of grouping in Section \ref{sec:grouping}, $\sigma \in \Sigma_n$ can be
understood as a function.
For example, in the case that $x = (x_1,x_2,x_3,x_4)$ 
and $\sigma = {1~2~3~4 \choose 2~4~3~1}
\in \Sigma_4$, $g_\sigma(x) = (x_2, x_4, x_3, x_1)$.
As mentioned in Section \ref{sec:grouping}, $\Sigma_n$ is a group and
each $\sigma$ is an isometry.

If $X$ is a real space, we can use the Euclidean distance $E$.
In that case, induced quotient metric on $X/\mathcal{G}$
is defined as follows:
\begin{displaymath}
\bar{d}(\bar{x},\bar{y}) := \min_{\sigma \in \Sigma_n}
E(x,\sigma(y))
\end{displaymath}

\noindent Figure~\ref{fig:ex_symmetric} will be helpful to understand
the quotient metric space for this case.

Induced quotient crossover can also be defined as in Definition \ref{def:iqx}.
Because it uses permutation,
it can be performed in $O(n^3)$ time by the Hungarian method similarly to
groupings.

\begin{figure}
\centering
$$X = \mathbb{R}^3$$
$$x = (1,4,5), y = (3,0,6) \in X$$
\begin{tabular}{|c|c|c|}
\hline
$\Sigma_3$ & $\bar{y}$ & $H(x,\sigma(y))$ \\
\hline
$\sigma_1 = (1,2,3)$ & $\sigma_1(y) = (3,0,6)$ & $E(x,\sigma_1(y)) = \sqrt{21}$
\\
$\sigma_2 = (1,3,2)$ & $\sigma_2(y) = (3,6,0)$ & $E(x,\sigma_2(y)) = \sqrt{33}$
\\
$\sigma_3 = (2,1,3)$ & $\sigma_3(y) = (0,3,6)$ & $E(x,\sigma_3(y)) = \sqrt{3}$
\\
$\sigma_4 = (2,3,1)$ & $\sigma_4(y) = (0,6,3)$ & $E(x,\sigma_4(y)) = 3$
\\
$\sigma_5 = (3,1,2)$ & $\sigma_5(y) = (6,3,0)$ & $E(x,\sigma_5(y)) = \sqrt{51}$
\\
$\sigma_6 = (3,2,1)$ & $\sigma_6(y) = (6,0,3)$ & $E(x,\sigma_6(y)) = 3\sqrt{5}$
\\
\hline
\end{tabular}
$$y^*=\sigma_3(y) = (0,3,6)$$
$$\bar{d}(\bar{x}, \bar{y}) = \sqrt{3}$$
\caption[]{An example of symmetric function under the Euclidean distance}
\label{fig:ex_symmetric}
\end{figure}

Summary for this Euclidean case is as follows:\\

\begin{tabular}{|l|l|}
\hline
Example & Symmetric functions on real space\\ \hline \hline
Genotype space $X$ & $\mathbb{R}^n$ \\ \hline
Isometry group $\mathcal{G}$ & $\Sigma_n$: a set of all permutations \\
inducing phenotype space $X/\mathcal{G}$ & of length $n$\\ \hline
Metric $d$ on $X$ & Euclidean distance $E$ \\ \hline
Original geometric crossover & traditional crossover on real vectors \\ \hline
Induced quotient crossover & rearranging before traditional crossover \\
& (newly proposed in this study) \\
\hline
\end{tabular}
\\

On the other hand, if $X$ is a discrete space as in binary or $k$-ary encoding, 
we can use the Hamming distance.
Then, induced quotient metric on $X/\mathcal{G}$
is defined as follows:
\begin{displaymath}
\bar{d}(\bar{x},\bar{y}) := \min_{\sigma \in \Sigma_n}
H(x,\sigma(y))
\end{displaymath}

\noindent Induced quotient crossover for this case can also
be performed in $O(n^3)$ time by the Hungarian method.
In sum, we have: \\

\begin{tabular}{|l|l|}
\hline
Example & Symmetric functions on discrete space \\ \hline \hline
Genotype space $X$ & $\{0,1\}^n$ \\ \hline
Isometry group $\mathcal{G}$ & $\Sigma_n$: a set of all permutations\\
inducing phenotype space $X/\mathcal{G}$ & of length $n$ \\ \hline
Metric $d$ on $X$ & Hamming distance $H$ \\ \hline
Original geometric crossover & traditional crossover on binary or $k$-ary vectors \\ \hline
Induced quotient crossover & rearranging before traditional crossover \\
& (newly proposed in this study) \\
\hline
\end{tabular}
\\

\subsection{Circular Permutations}
\label{sec:permutation}

Here we consider the case that solutions of a problem
are represented as circular permutations such as
traveling salesman problem (TSP).
Gluing head and tail of the permutation obtains a circular
permutation. Circular permutations cannot be represented directly.
They are typically represented with simple
permutations.
Then each circular permutation is
represented by more than one permutation.
For example, permutations $(1,2,3)$, $(2,3,1)$, and $(3,1,2)$
represent the same phenotype, i.e., circular permutation.
In such problem, the genotype space is 
a set of permutations
and the phenotype space is a set of
circular permutations.
We can consider this problem in view of
genotype-phenotype mapping using the concept of
quotient space.

Let $\Sigma_n$ be a set of all permutations with length $n$.
A function $s_k: \Sigma_n \rightarrow \Sigma_n$ is defined by
$k$-step circular shift operation to right.
For example, $s_2(1,2,3) = (2,3,1)$.
A set of all shift operations $S_n = \{s_k: k=0, 1, 2, \ldots, n-1\}$
is a group. And it is easy to check that 
each $s_k$ is an isometry on $\Sigma_n$.
If $\Sigma_n$ has a metric, $\Sigma_n/S$ has an induced quotient
metric by Proposition \ref{prop:metric}.

Now we consider various distances for permutation encoding.
The most typical distance is the Hamming distance $H$.
Under the Hamming distance, it is known that cycle crossover 
is geometric \cite{MorP05-2-p}. In this case, quotient metric is defined as follows:
\begin{displaymath}
\bar{d}(\bar{x},\bar{y}) := \min_{s \in S_n}
H(x,s(y)).
\end{displaymath}

\noindent An example case is shown in
Figure~\ref{fig:ex_permutation}.

Then we can define induced quotient crossover.

\begin{Def}[Position-independent cycle crossover]
Normalize the second parent to the first under the Hamming distance $H$.
Then, do the cycle crossover using the first parent and the normalized
second parent.
\end{Def}

\noindent Normalizing the second parent takes $O(n)$ time because
the equivalence class of the second parent has exactly $n$ elements
by shift operations. 

\begin{figure}
\centering
$$X = \Sigma_6$$
$$x = (2,4,5,1,6,3), y = (4,6,1,5,3,2) \in X$$
\begin{tabular}{|c|c|c|}
\hline
$S_6$ & $\bar{y}$ & $H(x,s(y))$ \\
\hline
$s_0$ & $s_0(y) = (4,6,1,5,3,2)$ & $H(x,s_0(y)) = 6$
\\
$s_1$ & $s_1(y) = (2,4,6,1,5,3)$ & $H(x,s_1(y)) = 2$
\\
$s_2$ & $s_2(y) = (3,2,4,6,1,5)$ & $H(x,s_2(y)) = 6$
\\
$s_3$ & $s_3(y) = (5,3,2,4,6,1)$ & $H(x,s_3(y)) = 5$
\\
$s_4$ & $s_4(y) = (1,5,3,2,4,6)$ & $H(x,s_4(y)) = 6$
\\
$s_5$ & $s_5(y) = (6,1,5,3,2,4)$ & $H(x,s_5(y)) = 5$
\\
\hline
\end{tabular}
$$y^*=s_1(y) = (2,4,6,1,5,3)$$
$$\bar{d}(\bar{x}, \bar{y}) = 2$$
\caption[]{An Example of circular permutation under the Hamming distance}
\label{fig:ex_permutation}
\end{figure}

Cycle crossover is also geometric under the swap distance.
The induced quotient crossover can be defined in a similar way
using the swap distance instead of the Hamming distance.
Summary for the case of applying the cycle crossover is as follows: \\

\begin{tabular}{|l|l|}
\hline
Example & circular permutations \\ \hline \hline
Genotype space $X$ & $\Sigma_n$ \\ \hline
Isometry group $\mathcal{G}$ & $S_n$: a set of all shift\\
inducing phenotype space $X/\mathcal{G}$ & operations \\ \hline
Metric $d$ on $X$ & Hamming distance $H$ (or swap distance)\\ \hline
Original geometric crossover & cycle crossover \\
\hline
Induced quotient crossover & rearranging before cycle crossover \\
& (newly proposed in this study) \\
\hline
\end{tabular}
\\

On the other hand,
we can use another well-known distance for $\Sigma_n$ - {\em reversal distance}.
Its neighborhood structure is the one
based on the 2-opt move. 
The reversal move selects any two points along the permutation then reverses
the subsequence between these points.
This move induces a graphic distance
between circular permutations: the minimum number of reversals to transform one 
circular permutation
into the other. 
The geometric crossover associated with this
distance belongs to the family of sorting crossovers \cite{MorP05-2-p}: it picks
offspring on the minimum sorting trajectory between parent circular
permutations sorted by reversals. 

\begin{Def}[Position-independent sorting-by-reversals crossover]
Normalize the second parent to the first under the graphic distance.
Then, do the crossover based on sorting by reversals for permutation
using the first parent and the normalized
second parent.\\
\end{Def}

\begin{tabular}{|l|l|}
\hline
Example & circular permutations \\ \hline \hline
Genotype space $X$ & $\Sigma_n$ \\ \hline
Isometry group $\mathcal{G}$ & $S_n$: a set of all shift\\
inducing phenotype space $X/\mathcal{G}$ & operations \\ \hline
Metric $d$ on $X$ & reversal distance \\ \hline
Original geometric crossover & sorting-by-reversals crossover \\
& for permutations \cite{MorP05-2-p} \\
\hline
Induced quotient crossover & sorting-by-reversals crossover \\
& for circular permutation \cite{MorP05-2-p} \\
\hline
\end{tabular} \\

There is a problem in implementing
the geometric crossover under the reversal distance.
Sorting linear or circular permutations by reversals is NP-hard \cite{Caprara97, Solo03}. So, the
geometric crossover under the reversal distance cannot be implemented
efficiently.
Nevertheless,
this example of quotient geometric crossover illustrates how to
obtain a geometric crossover for a transformed representation
(circular permutation) starting from a geometric crossover for the
original representation (permutation). So in this case
quotient geometric crossover is used as a tool to build a new
crossover for a derivative representation from a known geometric
crossover for the original representation. From \cite{MorP05-2-p}, we
know that the sorting-by-reversals crossover for permutations
is an excellent crossover for TSP. In future work we want to test
the sorting-by-reversals crossover for circular permutations. Since
they are a direct representation, we expect it to
perform even better.

\subsection{Sequences}
\label{sec:sequences}
An application in this subsection is not exactly fitted
to a quotient framework by the isometry subgroup
like applications introduced earlier.
However, we present this application because it
follows the quotient approach
except that the equivalence relation is not
from an isometry subgroup.

We recast alignment before recombination in
variable-length sequences as a consequence of quotient geometric
crossover. Consider the case that we use
{\em stretched} sequences as genotypes of sequences.
Stretched sequences mean sequences created by
interleaving `-' anywhere and in any number in the
sequences.
We can define a relation $\sim$ on stretched sequences: {\em each
stretched sequence belongs to the class of its unstretched version}.
Then, we can easily check that the relation $\sim$ is an equivalence relation.

In \cite{moraglio2006a-p}, Moraglio et al. have applied
geometric crossover to variable-length sequences.
The distance for
variable-length sequences they used there is the {\em edit distance}
$LD$\footnote{The notation $LD$ comes from
{\em Levenshtein distance} that is another name of edit distance.}:
the minimum number of insertion, deletion, and replacement of single
character to transform one sequence into the other. The geometric
crossover associated with this distance is proposed in
\cite{moraglio2006a-p}. It is called {\em homologous
geometric crossover}: two sequences are aligned optimally before
recombination. Alignment here means allowing parent sequences to be
stretched to match better with each other.
Two parent stretched sequences are aligned by interleaving or
removing `-' to create two stretched sequences of the same length that
have minimal Hamming distance. For example, if we want to recombine
{\tt agcacaca} and {\tt acacacta}, we need to align them optimally
first: {\tt agcacac-a} and {\tt a-cacacta}. Notice that the Hamming
distance between the aligned sequences is less than the Hamming distance
between the non-aligned sequences.
After the optimal alignment, one does the normal crossover
and produce a new {\em stretched} sequence. The
offspring is obtained by removing `-', so by unstretching the sequence.

From \cite{moraglio2006a-p}, we can easily check that
edit distance for sequences is a metric and hence
the phenotype space - the space of variable-length sequences -
is a quotient metric space.
In fact, the edit distance corresponds to a quotient metric and the homologous
geometric crossover corresponds to induced quotient crossover.
Suppose that we deal with only genotypes, i.e., stretched sequences.
We leave offspring produced by homologous crossover just stretched -
not removing `-'. Then the offspring exactly lies on quotient line segment.
So the crossover is a quotient geometric crossover in terms of
stretched sequences. In sum, we have: \\

\begin{tabular}{|l|l|}
\hline
Example & sequences \\ \hline \hline
Genotype space $X$ & stretched sequences \\ \hline
Equivalence relation $\sim$ & stretched sequences \\
inducing phenotype space $X/\!\!\sim$ & with the same unstretched sequence \\ \hline
Metric $d$ on $X$ & edit distance \\ \hline
Original geometric crossover & traditional crossover on stretched sequences \\
\hline
Induced quotient crossover & homologous crossover \cite{moraglio2006a-p} \\
\hline
\end{tabular}
\\

Phenotypes are variable-length sequences that are directly
representable. So in this case the quotient geometric crossover is
not used to search a non-directly representable space (phenotypes)
through an auxiliary directly representable space (genotypes). The
benefit of applying the quotient geometric crossover on
variable-length sequences is that the homologous crossover over
sequences $GX_P$ is naturally understood as a transformation of
the geometric crossover $GX_G$ over stretched sequences $G$ rather
than a crossover acting directly on sequences $P$. This is because
the notion of optimal alignment is inherently defined on stretched
sequences and not on simple sequences. In \cite{moraglio2006a-p},
Moraglio et al. have tested the homologous crossover on the protein motif
discovery problem. In future work we want to study how the optimal
alignment transformation affects the fitness landscape associated
with geometric crossover with and without alignment.

\section{Concluding Remarks}
\label{sec:conclusions}

In this paper we have mathematically analyzed genotype and phenotype spaces by
introducing the notion of quotient space.
Phenotype space can be regarded as quotient space by a genotype-phenotype mapping.
Geometric crossovers has the advantage in that they can also be formally defined once a
distance is defined. Owing to this advantage we can connect a solution space - as
a metric space - and crossovers. Moreover, geometric crossover based on the
appropriate distance of a space reflects properties of given space.
We introduced quotient metric on phenotype space. 
Since the quotient metric
is a part of the
phenotype space structure, the geometric crossover by the metric
reflects the properties of phenotype space more effectively
than the original geometric crossover.

As shown in application examples,
quotient geometric crossover is not only theoretically significant
but also has a practical effect of making search more effective
by reducing the search space or removing the inherent bias.
In the example of grouping, we newly reinterpreted geometric crossover \cite{Mor07}
that was previously proposed by the
authors to be theoretically complete.
In the examples of graphs, symmetric functions, and circular permutations,
we induced new crossovers better tailored to phenotype space
using the proposed methodology.
In the example of sequence, we successfully analyzed previous study \cite{moraglio2006a-p}
in view of our quotient theory though it is slightly escaped from
the framework we presented.

In future work, we will test the proposed induced quotient crossovers
in solving the problems using genetic algorithms.
Also, more examples and applications for each example case
are left for future study.

{
\bibliographystyle{plain}
\bibliography{ref}
}

\end{document}